\documentclass[letterpaper, 10 pt, conference]{ieeeconf}  

\usepackage{amsmath}
\usepackage{amsfonts}
\newtheorem{theorem}{Theorem}
\newtheorem{lemma}{Lemma}
\usepackage{graphicx}
\usepackage{multirow}
\usepackage{xcolor}
\usepackage{comment}
\usepackage{flushend}

\IEEEoverridecommandlockouts                       

\title{\LARGE \bf
Adaptive Ankle Torque Control for Bipedal Humanoid Walking on Surfaces with Unknown Horizontal and Vertical Motion
}

\author{Jacob Stewart$^{*1}$, I-Chia Chang$^{*2}$, Yan Gu$^{\dagger2}$, and Petros A. Ioannou$^{1}$  
    \thanks{$^{*}$These authors contributed equally. $^{^{\dagger}}$Corresponding author.}
    \thanks{$^{1}$Jacob Stewart and Petros A. Ioannou are with the Viterbi School of Engineering, University of Southern California, Los Angeles, CA 90007.
            {\tt \footnotesize \{jacobmst, ioannou\}@usc.edu.}}%
    \thanks{$^{2}$I-Chia Chang and Yan Gu are with the School of Mechanical Engineering, Purdue University,
            West Lafayette, IN 47907, USA. E-mails: {\tt \footnotesize \{chang970, yangu\}@purdue.edu.}}
}

\begin{document}

\maketitle
\thispagestyle{plain}
\pagestyle{plain}

\begin{abstract}
Achieving stable bipedal walking on surfaces with unknown motion remains a challenging control problem due to the hybrid, time-varying, partially unknown dynamics of the robot and the difficulty of accurate state and surface motion estimation. 
Surface motion imposes uncertainty on both system parameters and non-homogeneous disturbance in the walking robot dynamics.
In this paper, we design an adaptive ankle torque controller to simultaneously address these two uncertainties and propose a step-length planner to minimize the required control torque. 
Typically, an adaptive controller is used for a continuous system. To apply adaptive control on a hybrid system such as a walking robot, an intermediate command profile is introduced to ensure a continuous error system. 
Simulations on a planar bipedal robot, along with comparisons against a baseline controller, demonstrate that the proposed approach effectively ensures stable walking and accurate tracking under unknown, time-varying disturbances.
\end{abstract}

\section{INTRODUCTION}

Bipedal robots have shown their ability to walk stably on stationary surfaces using either model-based \cite{973365, xiong20223} or data-driven \cite{dai2023data, castillo2024data} methods. 
However, locomotion control on moving surfaces remains challenging since surface motion induces unknown, time-varying disturbance in the nonlinear and hybrid dynamics of a walking robot \cite{gao2023time}.
One challenge of this problem comes from the difficulty of state estimation, where the assumption of the stationary foot-ground contact points does not hold~\cite{he2024legged}. 
In terms of controller design, the surface motion can be treated as a known disturbance input \cite{gao2023time} or an unknown external disturbance \cite{gu2024walking}. It can also be assumed to have bounded acceleration in the vertical \cite{iqbal2023asymptotic} or horizontal directions \cite{chen2024contingency}.
Yet, previous control methods that individually deal with vertical~\cite{iqbal2020provably, iqbal2023analytical,iqbal2023asymptotic} or horizontal~\cite{gao2023time, chen2024contingency} motions cannot be directly extended to address unknown motions in both directions.
In the reduced-order dynamics model of a robot walking on moving surfaces, the vertical surface motion induces parametric model uncertainty~\cite{iqbal2023asymptotic}, while the horizontal surface motion causes non-homogeneous disturbance~\cite{gao2023time}.
Previous work addressing vertical ground motions either assumes an accurately known surface movement profile~\cite{iqbal2020provably, iqbal2023analytical} or ignores the forcing term in the tracking error dynamics induced by horizontal ground motions~\cite{iqbal2023asymptotic}, which may not be effective for minimizing tracking errors.
Meanwhile, existing control approaches for horizontal ground movements~\cite{gao2023time} assume the surface exhibits an exactly periodic, accurately known motion, and do not explicitly address the parametric uncertainties~\cite{chen2024contingency} or disturbances caused by unknown ground translations. 
Currently, controllers that explicitly treat both unknown vertical and horizontal surface motions are missing. 

Adaptive control has been used in a variety of industrial applications to reject disturbances with unknown parameters affecting systems with either known or unknown dynamics.
In the case of a bipedal robot walking on a moving surface, there is often very limited knowledge of the characteristics of the surface movement beyond it is a sum of sinusoidal components with time-varying parameters \cite{han2021data,tannuri2003estimating,zhang2020ship}.
Though some adaptive control schemes can reject disturbances with very minimal knowledge of the system dynamics \cite{jafari2017overparameterized,kinney2009robust}, adaptive controllers which use the available knowledge of the system dynamics typically perform more effectively and have a lower computation complexity \cite{landau2013benchmark}.
These methods are robust to measurement noise and model uncertainty \cite{jafari2016contMimo} and can be tuned to be effective in the presence of time variations \cite{stewart2024ltiSystem,stewart2024ltvSystem}.
However, adaptive controllers are typically suitable for continuous systems, and thus may not be directly applicable to walking robots, whose dynamics are inherently hybrid~\cite{xiong20223} due to the robot's feet making and breaking contact with the ground.
Although adaptive control has recently been extended to hybrid systems including walking robots~\cite{gu2021adaptive}, the approach focuses on time-invariant robot dynamics rather than time-varying behaviors associated with walking during ground motion. 

This paper introduces an adaptive ankle torque control approach that achieves stable bipedal walking during unknown horizontal and vertical ground motions.
The main contributions of this study are:
\begin{enumerate}
    \item [(a)] Introduction of a fictitious hybrid error system based on the modification of the desired robot state, which, despite the hybrid nature of walking robot dynamics, renders a continuous error system, allowing the direct application of existing adaptive controllers typically designed for continuous systems. This method expands the range of applications of adaptive control.
    \item [(b)] Construction of an adaptive ankle torque controller that explicitly treats the unknown horizontal and vertical ground motions. This controller not only ensures walking stability but also enables accurate position tracking within the continuous phase of each walking step. Such a performance is not achievable with previous footstep controllers \cite{gao2023time, iqbal2023asymptotic} because they leave the continuous-phase robot dynamics uncontrolled and unstable. 
    \item [(c)] Development of a discrete footstep controller for the proposed fictitious error system, which stabilizes the system and greatly reduces the demanded magnitude of the proposed ankle torque controller. 
    \item [(d)] Validation of the proposed control approach through simulations of a seven-link bipedal robot under various surface motion uncertainties.
\end{enumerate}



\vspace{-0.05 in}
\section{PROBLEM FORMULATION}

\begin{figure}
    \centering
    \includegraphics[width= 0.45\textwidth]
    {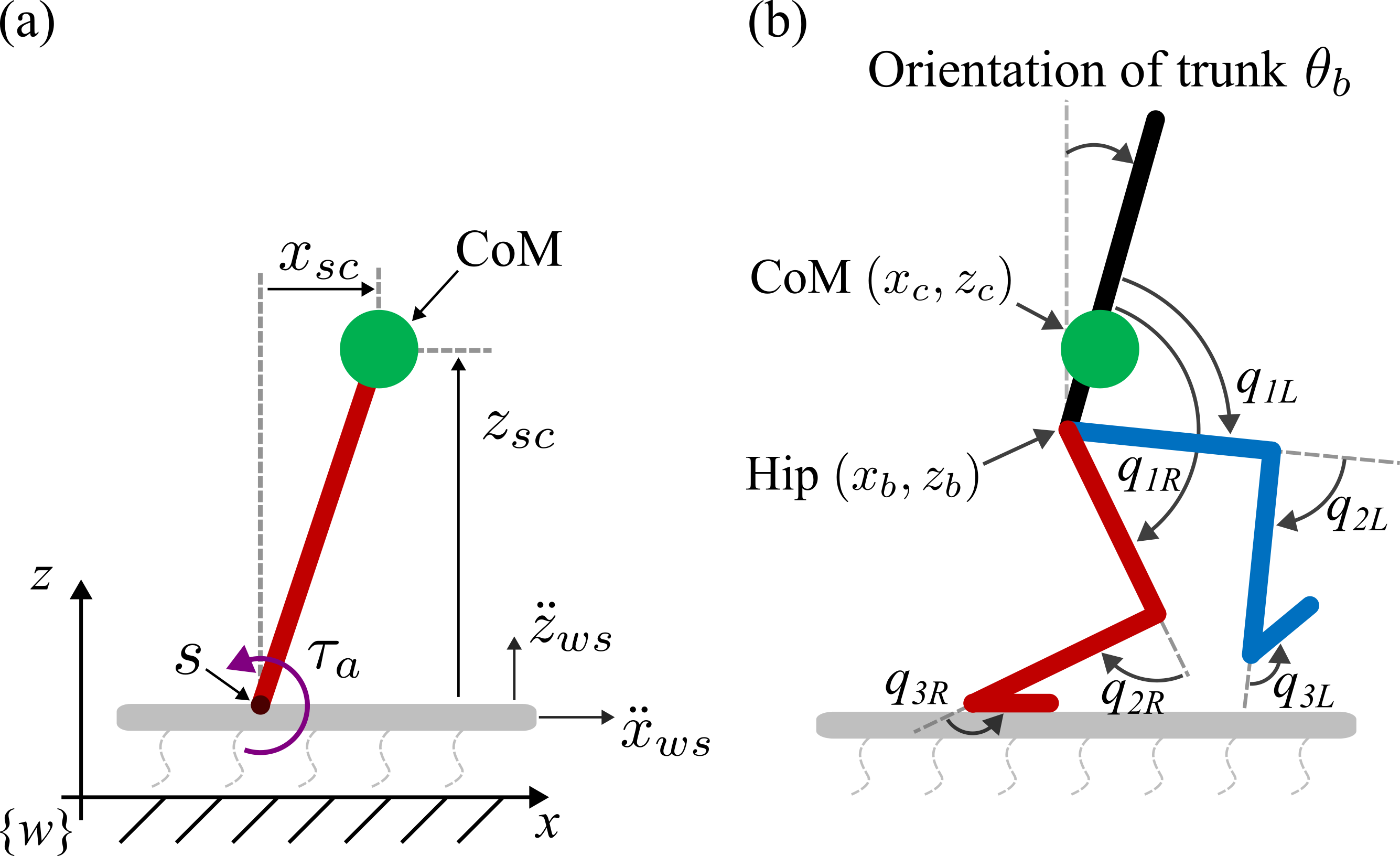}
    \vspace{-0.1 in}
    \caption{(a) A linear inverted pendulum model on a platform moving along the $x$ and $z$ axes of the world frame \{$w$\}. (b) A seven-link robot whose state is $\mathbf{q} = [\mathbf{x}_{base}^T, \mathbf{q}_{j}^T]^T$, with $\mathbf{x}_{base} = [x_b, z_b, \theta_b]^T$ and $\mathbf{q}_j = [q_{1L}, q_{2L}, q_{3L}, q_{1R}, q_{2R}, q_{3R}]^T$. The variables $x_b$, $z_b$, $\theta_b$, and the elements of $\mathbf{q}_j$ are defined as illustrated in subplot (b). }
    \label{fig: lip and seven link robot} 
    \vspace{-0.1 in}
\end{figure}

This paper focuses on designing an adaptive control approach that ensures stable bipedal walking during unknown ground motions in both horizontal and vertical directions. 

To address the hybrid nature of walking, we explicitly consider the discrete events of foot touchdowns. The foot touchdown times are denoted $t_k$ indexed by $k\in\mathbb{N}$.
The instants immediately before and after touchdown are denoted $t_k^-$ and $t_k^+$, respectively.
Without loss of generality, the touchdown times are selected as $t_k = (k-0.5)T_s$,
where $T_s$ is the user-specified step period.

Based on the assumption that the mass of the robot is concentrated at the center of mass (CoM) and the angular momentum about CoM remains constant \cite{973365, iqbal2023asymptotic}, we model the CoM dynamics using the disturbed variable-length inverted pendulum, as shown in Fig.\ref{fig: lip and seven link robot}(a).
As illustrated in the figure, $x_{sc}(t)$ and $z_{sc}(t)$ are the known horizontal and vertical position of the CoM with respect to the support point $s$ that is fixed to the moving surface.

Using $\mathbf{x}(t) := [ x_{sc}(t), \dot{x}_{sc}(t)]^T$ to denote the state of the horizontal CoM dynamics, and applying Lagrange's method,
we obtain the disturbed variable-length inverted pendulum model as
\begin{equation}
\label{sys}
\begin{aligned} 
        \dot{\mathbf{x}}(t)
        =&
        \begin{bmatrix} 0 & 1 \\ 0 & 0  \end{bmatrix}
        \mathbf{x}(t)
        +
        \mathbf{B}_2
        f\left( t, \mathbf{x}(t) , \tau_a(t) \right),&& \text{if}~t \neq t_k^-
        \\
        \mathbf{x}(t_k^+) =& \mathbf{x}(t_k^-) - \mathbf{B}_1 u(t_k^-),&& \text{if}~t = t_k^-
\end{aligned}
\end{equation}
where 
$\mathbf{B}_1 = [1, 0]^T$,
$\mathbf{B}_2 = [0, 1]^T$, and
the function $f$ is
\begin{align*}
    \begin{split}
        f \left( t, \mathbf{x}(t) , \tau_a(t) \right)
        =
        \tfrac{ g + \ddot{z}_{ws}(t) }{ z_{sc}(t) } x_{sc}(t)
        -
        \ddot{x}_{ws}(t)
        -
        \tfrac{\tau_a(t)}{ m z_{sc}(t) }.
    \end{split}
\end{align*}
Here, $m$ is the mass of the robot and $g$ is the magnitude of the gravitational acceleration.
$\ddot{x}_{ws}(t)$ and $\ddot{z}_{ws}(t)$ are the bounded unknown accelerations of the support point $s$ with respect to an inertial world frame in horizontal and vertical directions, respectively. 
This equation reveals that $\ddot{x}_{ws}(t)$ induces non-homogeneous disturbance while $\ddot{z}_{ws}(t)$ causes parametric uncertainty.
The ankle torque $\tau_a(t)$ and the step length $u(t_k^-)$ are the control signals to be designed.

Suppose that the vertical position $z_{sc}(t)$ is regulated (e.g., by a full-dimensional robot controller~\cite{gao2023time,xiong20223}) to be within some bounded value $z_b$ of the desired constant vertical position $z_{sc}^d$ for all time; i.e.,  $|z_{sc}(t) - z_{sc}^d| < z_b < z_{sc}^d$ for all $t>0$. Note that the condition $z_b < z_{sc}^d$ is necessary to ensure $z_{sc}(t) > 0$ for all $t>0$.

Without loss of generality and for simplicity, we consider that the robot is initially standing at rest; i.e., the initial condition of the system is
$\mathbf{x}(0) = [0, 0]^T$ and $\mathbf{z}(0) = [ z_{sc}^d, 0 ]^T$, where $\mathbf{z}(t) =: [z_{sc}(t), \dot{z}_{sc}(t)]^T$ is a vector composed of the vertical position and velocity.

This study focuses on treating surface accelerations that are a sum of sinusoidal terms with time-varying characteristics; i.e., the disturbance induced by surface acceleration takes the following general form
\begin{align}
    \begin{split} \label{distAccel}
        \ddot{x}_{ws}(t)
        =&
        a_{x0}(t)
        +
        \sum_{i=1}^{n_x} a_{xi}(t) \sin{\left( \int_0^t \omega_{xi}(\tau)\partial\tau +\varphi_{xi} \right)}
        \\
        \ddot{z}_{ws}(t)
        =&
        a_{z0}(t)
        +
        \sum_{i=1}^{n_z} a_{zi}(t) \sin{\left( \int_0^t \omega_{zi}(\tau)\partial\tau +\varphi_{zi} \right)}
    \end{split}
\end{align}
where $n_x$, $n_z$ are the numbers of sinusoidal terms, 
$a_{x0}(t)$, $a_{z0}(t)$ are the bias terms,
$a_{xi}(t)$, $a_{zi}(t)$ are the disturbance amplitudes, 
$\omega_{xi}(t)$, $\omega_{zi}(t)$ are the disturbance frequencies, and 
$\varphi_{xi}$, $\varphi_{zi}$ are the disturbance phases.
Assume the parameters of the disturbance are slowly varying with respect to time; i.e., there exists a parameter $\mu\ge0$ such that 
$\sup_t |a_{x0}^{(k)}(t)| \le \mu$, 
$\sup_t |a_{z0}^{(k)}(t)| \le \mu$, 
$\sup_t |a_{xi}^{(k)}(t)| \le \mu$, 
$\sup_t |\omega_{xi}^{(k)}(t)| \le \mu$, 
$\sup_t |a_{zj}^{(k)}(t)| \le \mu$, 
$\sup_t |\omega_{zj}^{(k)}(t)| \le \mu$ 
for all derivative orders 
$k\in\mathbb{N}$ 
and all indices 
$i\in\{1,\dots,n_x\}$, $j\in\{1,\dots,n_z\}$.
Also assume the amplitudes are such that the disturbances are uniformly bounded i.e. there exists $\bar{x},\bar{z}<\infty$ such that $\sup_t|\ddot{x}_{ws}(t)| \le \bar{x}$ and $\sup_t|\ddot{z}_{ws}(t)| \le \bar{z}$.

To minimize the requested ankle torque, the desired CoM motion profile is designed to be a hyperbolic profile of a linear inverted pendulum model without ankle torque actuation. Specifically, the desired hyperbolic, piecewise continuous motion profile is the solution of the following hybrid system
\begin{equation}
\label{desProfile}
\begin{aligned} 
        \dot{\mathbf{x}}^d(t)
        =&
        \mathbf{A}_\lambda
        \mathbf{x}^d(t),&&\text{if}~t \neq t_k^-
        \\
        \mathbf{x}^d(t_k^+)
        =&
        \mathbf{x}^d(t_k^-)
        -
        \mathbf{B}_1
        T_sv^d,&&\text{if}~t = t_k^-
\end{aligned}
\end{equation}
where $\mathbf{x}^d(t) =: [x_{sc}^d(t), \dot{x}_{sc}^d(t)]^T$ is the desired trajectory of $\mathbf{x}(t)$,
and
$\mathbf{A}_\lambda = \begin{bmatrix} 0 & 1 \\ \lambda^2 & 0 \end{bmatrix}$ with 
$\lambda = \sqrt{\frac{g}{z_{sc}^d}}$.
To enforce the desired walking velocity $v^d$ as the average value of $\dot{x}_{sc}^d(t)$ within each walking step, the initial condition is set to
$\mathbf{x}^d(0) = [0, \frac{v^d T_s \lambda}{2 \sinh{(0.5 \lambda T_s)}}]^T$.

The control objective is to design the two control inputs, the ankle torque $\tau_a(t)$ and footstep length $u(t_k^-)$ ($k \in \mathbb{N}$), such that (a) the ankle torque is less than a known limit $\bar{\tau}_{a}$ and (b) the velocity tracking error $x^d_{sc}(t) - x_{sc}(t)$ converges to a small residual set as $t \rightarrow \infty$,
in the presence of the unknown disturbances $\ddot{x}_{ws}(t)$ and $\ddot{z}_{ws}(t)$.


\section{METHOD}
This section presents the proposed control approach that explicitly addresses unknown vertical and horizontal ground motions to achieve stable bipedal walking.

The proposed control approach is constructed based on the forced pendulum model in~\eqref{sys}, and comprises three primary elements: discrete footstep planner, continuous stabilizing controller, and adaptive controller.
The design of the proposed approach begins with the modification of the desired hybrid trajectory of the pendulum state, producing both a fictitious hybrid error system and a continuous actual error system.
A discrete footstep planner is then synthesized to asymptotically stabilize the fictitious system, while the continuous actual error system is treated by the continuous stabilizing and adaptive controllers.
The footstep planner also helps keep the required ankle torque small for the two continuous controllers.


\vspace{-0.05 in}
\subsection{Discrete Footstep Planner Based on Commanded State}

To introduce a continuous error system, we modify the desired state $\mathbf{x}^d(t)$ to generate a commanded state trajectory, $\mathbf{x}^c(t)$, which starts at the robot's true initial state and converges to the desired state trajectory, $\mathbf{x}^d(t)$. 
The use of $\mathbf{x}^c(t)$ prevents an overly large state error when the robot starts from rest and prevents jumps in the error state, $\mathbf{x}^c(t)-\mathbf{x}(t)$, immediately after the discrete foot touchdown event.
This lowers the demand for significant control input. 
In addition, this design choice makes the dynamics of $\mathbf{x}^c(t)-\mathbf{x}(t)$ completely continuous, which is the basis for the designs of the two proposed continuous controllers.

To drive the commanded state $\mathbf{x}^c(t)$ toward the desired state $\mathbf{x}^d(t)$ using the step length $u$, we design the dynamics of the commanded state $\mathbf{x}^c(t)$ as the following hybrid system
\begin{equation}
\begin{aligned}
\label{cmdProfile}
        \dot{\mathbf{x}}^c(t)
        =&
        \mathbf{A}_\lambda
        \mathbf{x}^c(t),
        && \text{if}~t \neq t_k^-
        \\
        \mathbf{x}^c(t_{k}^+)
        =&
        \mathbf{x}^c(t_{k}^-)
        -
        \mathbf{B}_1
        u(t_{k}^-), && \text{if}~t = t_k^-
\end{aligned}
\end{equation}
where $ \mathbf{x}^c(t) =: \begin{bmatrix} x_{sc}^c(t), \dot{x}_{sc}^c(t) \end{bmatrix}^T$, with the initial condition $\mathbf{x}^c(0) = \mathbf{x}(0)$.
Let the error between the desired profile and commanded profile be defined by
\begin{align} \label{cmdToDesError}
    \mathbf{e}^c(t) := \begin{bmatrix}
        e_{sc}^c(t) \\ \dot{e}_{sc}^c(t)
    \end{bmatrix} := \mathbf{x}^d(t) - \mathbf{x}^c(t)
\end{align}
To drive $\mathbf{x}^c(t)$ to $\mathbf{x}^d(t)$, we propose a Linear-Quadratic Regulator (LQR) step length controller, which drives the distance traveled over one step period to $T_s v^d$ and is described by
\begin{align}
    u(t_{k+1}^-) = T_sv^d - \mathbf{K} ( \mathbf{A}_s - \mathbf{I} ) \mathbf{e}^c(t_k^+) \label{LQRsteplengthcontrollaw}
\end{align}
where
$\mathbf{K} = (\mathbf{R} + \mathbf{B}_s^{T} \mathbf{P}_s \mathbf{B}_s )^{-1} \mathbf{B}_s^{T} \mathbf{P}_s \mathbf{A}_s$, 
$\mathbf{A}_s = \text{exp}( \mathbf{A}_\lambda T_s )$, and
$\mathbf{B}_s = (\mathbf{A}_s - \mathbf{I})
    \mathbf{B}_1$.
The matrix $\mathbf{P}_s$ is the solution to the discrete-time algebraic Ricatti equation
$\mathbf{P}_s =
        \mathbf{A}_s^{T} \mathbf{P}_s \mathbf{A}_s
        +
        \mathbf{Q}
        -
        \left( \mathbf{A}_s^{T} \mathbf{P}_s \mathbf{B}_s \right)
        \left( \mathbf{R} + \mathbf{B}_s^{T} \mathbf{P}_s \mathbf{B}_s \right)^{-1}
        \left( \mathbf{B}_s^{T} \mathbf{P}_s \mathbf{A}_s \right)$.
The matrices $\mathbf{R}>0$ and $\mathbf{Q}\ge0$ are the LQR parameters chosen for good performance.
Then we have the following result.
\begin{lemma} \label{lem:stepCtrl}
    The step lengths chosen according to (\ref{LQRsteplengthcontrollaw}) drives $\mathbf{e}^c(t) \to 0$ and $\mathbf{x}^c(t)\to \mathbf{x}^d(t)$ as $t\to\infty$.
\end{lemma}
\textit{Proof: see the appendix.}

\vspace{-0.05 in}
\subsection{Continuous Stabilizing Controller Based on Actual State}
Prior to presenting the proposed adaptive controller, we first introduce an ankle torque controller that stabilizes the dynamics of the actual pendulum state error.
This actual error is defined as the discrepancy between the current actual state and the commanded state; i.e.,
$\mathbf{e}(t) := \begin{bmatrix} e_{sc}(t) , \dot{e}_{sc}(t) \end{bmatrix}^T
    := \mathbf{x}^c(t) - \mathbf{x}(t)$.
Based on the dynamics of the actual (\ref{sys}) and the commanded state (\ref{cmdProfile}), the error system is expressed as
\begin{align*}
        \dot{\mathbf{e}}(t)
        =&
        \begin{bmatrix}
            0 & 1 \\ \frac{\ddot{z}_{ws}(t)}{z_{sc}(t)} & 0
        \end{bmatrix}
        \mathbf{e}(t)
        +
        \mathbf{B}_2
        h\left( t , e_{sc}(t) , \tau_a(t) \right),&&\text{if}~t \neq t_k^-
        \\
        \mathbf{e}(t_k^+) =& \mathbf{e}(t_k^-),&&\text{if}~t = t_k^-
\end{align*}
where
\begin{align*}
    \begin{split}
        h ( t, e_{sc}(t) , & \tau_a(t) )
        =
        \tfrac{\tau_a(t)}{m z_{sc}(t)}
        +
        \tfrac{g}{z_{sc}(t)}e_{sc}(t)
        \\
        +&
        \left(
        \tfrac{g}{z_{sc}^d}
        -
        \tfrac{g}{z_{sc}(t)}
        \right)
        x_{sc}^c(t)
        +
        \ddot{x}_{ws}(t)
        -
        \tfrac{ x_{sc}^c(t) }{z_{sc}(t)}
        \ddot{z}_{ws}(t)
    \end{split}
\end{align*}
Observe that the error system between the current state and the commanded state is no longer a hybrid system with a state jump across $t = t_{k}$.
This ensures the error observed by the adaptive law is continuous.

To stabilize the continuous error dynamics, we design the ankle torque command as a combination of feed-forward and proportional-derivative (PD) feedback control terms:
\begin{align} \label{pdCtrl}
\begin{split}
    \tau_a(t)
    =
    m z_{sc}(t) &
    \left[  \left(  - \tfrac{g}{z_{sc}(t)} - k_p\right)e_{sc}(t) - k_d \dot{e}_{sc}(t)  \right.
    \\
    - & \left. 
    \left( \tfrac{g}{z_{sc}^d} - \tfrac{g}{z_{sc}(t)} \right)x_{sc}^c
    +
    k_p v(t) \right]
\end{split}
\end{align}
where $k_p$ and $k_d$ are the PD gains and $v(t)$ is an input from the proposed adaptive controller to be introduced in the next subsection.
Then, the closed-loop error dynamics is given by
\begin{align} \label{pdClosedLoop}
    \dot{\mathbf{e}}(t)
    =&
    \left( \mathbf{A} + \mathbf{\Delta}(t) \right)
    \mathbf{e}(t)
    +
    \mathbf{B} \left( v(t) - d(t) \right)
\end{align}
where
$
\mathbf{A}
=
\begin{bmatrix}
    0 & 1 \\ -k_p & -k_d
\end{bmatrix}
$,
$
\mathbf{\Delta}(t)
=
\begin{bmatrix} 0 & 0 \\ \tfrac{\ddot{z}_{ws}(t)}{z_{sc}(t)} & 0 \end{bmatrix}
$, and
$
\mathbf{B}
= [0, k_p]^T
$. 
and the unknown input disturbance $d(t)$ is given by
\begin{align} \label{dist}
    d(t)
    =
    \tfrac{-1}{k_p}
    \left(
    \ddot{x}_{ws}(t)
    -
    \tfrac{ x_{sc}^c(t) }{z_{sc}(t)}
    \ddot{z}_{ws}(t)
    \right).
\end{align}

The following lemma states that the proposed ankle torque control (\ref{pdCtrl}) with $v(t)=0$ can be used to stabilize the continuous error system provided that the unknown disturbance $\ddot{z}_{ws}(t)$, which manifests as model uncertainty in the closed-loop system, is small enough not to destabilize the system.

\begin{lemma} \label{lem:stable}
    Assume $k_p, k_d > 0$ can be chosen such that 
    $\sup_t \left| \frac{\ddot{z}_{ws}(t)}{z_{sc}(t)} \right| \le \frac{\bar{z}}{z_{sc}^d-z_b} < \frac{1}{2\|\mathbf{L}\|}$,
    with $\mathbf{L} = \mathbf{L}^{T} >0$ the solution to the Lyapunov equation
    $\mathbf{A}^{T} \mathbf{L} + \mathbf{L} \mathbf{A} = -\mathbf{I}$, where $\mathbf{I}$ is an identity matrix of a proper dimension.
    Then the controller (\ref{pdCtrl}) with $v(t) = 0$ drives $\mathbf{e}(t)$ to a residual set whose size is uniformly bounded and proportional to $\limsup_{t\to\infty} \left| \ddot{x}_{ws}(t) - \frac{ x_{sc}^c(t) }{z_{sc}(t)}\ddot{z}_{ws}(t) \right| \le \bar{x} + \frac{0.5T_sv^d}{2\|\mathbf{L}\|} < \infty$ .
\end{lemma}
\textit{Proof: see the appendix.}

\vspace{-0.05 in}
\subsection{Continuous Adaptive Controller}

The adaptive controller generating the control input $v(t)$ is comprised of a observer, a compensator, and an adaptive parameter estimator.
The objective of the adaptive controller is to reject the unknown disturbance $d(t)$ for the actual error dynamics in~\eqref{pdClosedLoop}.
The design of this controller is inspired by existing adaptive control ~\cite{stewart2024ltvSystem} for continuous systems that include the actual error dynamics in~\eqref{pdClosedLoop}.

\subsubsection{Observer design}
To build the observer, we introduce an observer state $\hat{\mathbf{e}}(t)$ whose dynamics does not directly depend on the disturbances \eqref{dist}, along with a time-varying variable $\zeta(t)$ that represents the effect of the disturbance (\ref{dist}) on the CoM position error $e_{sc}(t)$:
\begin{align}
    \begin{split} \label{observer}
        \dot{\hat{\mathbf{e}}}(t)
        =&
        \mathbf{A}\hat{\mathbf{e}}(t)
        +
        \mathbf{B}v(t) 
        \\
        \zeta(t)
        =& 
        \mathbf{C}
        \left( \mathbf{e}(t) - \hat{\mathbf{e}}(t) \right)
    \end{split}
\end{align}
where 
$\mathbf{C} = \begin{bmatrix} 1 & 0 \end{bmatrix}$.
Here, the disturbance-induced term $\zeta(t)$ is parameterized as~\cite{ioannou2012robust}:
\begin{align}
    \zeta(t) = \boldsymbol{\phi}^T(t)\boldsymbol{\theta}(t)
    +
    \kappa(t) \label{linearEstimator}
\end{align}
where $\boldsymbol{\phi}: \mathbb{R}^+\rightarrow \mathbb{R}^{n_{\phi}}$ is a regressor vector,
$\boldsymbol{\theta}: \mathbb{R}^+ \rightarrow \mathbb{R}^{n_{\phi}}$ is the unknown parameter vector with its $k^{\text{th}}$ ($k \in \{1,2,...,n_\phi\}$) element $\theta_k(t)$,
and $\kappa(t)$ is a bounded error term whose size is proportional to the parameter $\mu$ \cite{stewart2024ltvSystem}.

\subsubsection{Compensator design}
The observer in \eqref{observer} takes $v(t)$ as its input.
To construct $v(t)$ for ensuring a reasonable convergence rate of parameter estimation in practice, we utilize a parallel low-pass filter~\cite{jafari2016contMimo} in the form of
$\mathcal{L}[v(t)] = \sum_{k=1}^{n_\phi} \frac{\sigma^k}{(s+\sigma)^k}\mathcal{L}\left[\hat{\theta}_k(t)\zeta(t)\right](s)$.
Here, $\hat{\boldsymbol{\theta}}(t)$ is the estimated value of the unknown parameter $\boldsymbol{\theta}(t)$.
The bandwidth parameter $\sigma > 0$ and the compensator order $n_\phi$ are design parameters, which can be tuned for good performance. $n_\phi$ must be more than twice the number of sinusoidal components in the disturbance (\ref{dist}) plus 1.

The state-space form of the parallel low-pass filter that produces the control input $v(t)$ is given by
\begin{align}
    \begin{split} \label{compensator}
        \dot{\boldsymbol{\eta}}(t)
        =& 
        \mathbf{F}\boldsymbol{\eta}(t)
        +
        \hat{\boldsymbol{\theta}}(t)\zeta(t)
        \\
        v(t) =& \mathbf{H}\boldsymbol{\eta}(t)
    \end{split}
\end{align}
where 
$\boldsymbol{\eta} \in \mathbb{R}^{n_\phi}$ is the filter state, $\mathbf{H} := \sigma\begin{bmatrix} 1 & 0 & \cdots & 0 \end{bmatrix}$, and
$\mathbf{F} := \sigma(\mathbf{U} - \mathbf{I})$,
with $\mathbf{U}$ an upper diagonal shift matrix with 1's along the first upper diagonal and zeros everywhere else. 

\subsubsection{Parameter estimation}
To generate the regressor vector $\boldsymbol{\phi}(t)$ based on the known value of $\zeta(t)$, we utilize the following dynamical system~\cite{ioannou2012robust}
\begin{align}
    \begin{split}
        \begin{bmatrix}
            \dot{\mathbf{X}}(t) \\ \dot{\mathbf{Y}}(t)
        \end{bmatrix}
        =&
        \begin{bmatrix}
            \mathbf{A} & \mathbf{B}\mathbf{H} \\ \mathbf{0} & \mathbf{F}
        \end{bmatrix}
        \begin{bmatrix}
            \mathbf{X}(t) \\ \mathbf{Y}(t)
        \end{bmatrix}
        +
        \begin{bmatrix}
            \mathbf{0} \\ \zeta(t) \mathbf{I}
        \end{bmatrix}
        \\
        \boldsymbol{\phi}^T(t)
        =&
        \begin{bmatrix}
            \mathbf{C} & \mathbf{0}
        \end{bmatrix}
        \begin{bmatrix}
            \mathbf{X}(t) \\ \mathbf{Y}(t)
        \end{bmatrix}
    \end{split}
\end{align}
where $\mathbf{X} \in \mathbb{R}^{2\times n_\phi}$ and $\mathbf{Y} \in \mathbb{R}^{ n_\phi \times n_\phi}$ are matrix-valued states and $\mathbf{0}$ is a zero matrix/vector with a proper dimension.

The time variation of the unknown parameter vector $\boldsymbol{\theta}(t)$ motivates the generation of the parameter estimates $\hat{\boldsymbol{\theta}}(t)$ using the modified recursive-least-squares adaptive law with exponential forgetting, resetting, and projection described by
\begin{equation}
\label{updateLaw}
    \dot{\hat{\boldsymbol{\theta}}}(t)
        =
        \text{proj}\left( 
        \boldsymbol{\psi}(t)
        \right)
\end{equation}
where \text{proj}$(\cdot)$ represents a projection operator and $\boldsymbol{\psi}(t):=\alpha \mathbf{P}(t)\boldsymbol{\phi}(t)\epsilon(t)$.
Here,
$\mathbf{P}(t)$ is the covariance matrix,
$\epsilon(t)$ is the normalized estimation error designed as $\epsilon(t)
        =
        \tfrac{\zeta(t) - \boldsymbol{\phi}^T(t)\hat{\boldsymbol{\theta}}(t)}{1+\boldsymbol{\phi}^T(t)\mathbf{P}(t)\boldsymbol{\phi}(t)}$,
and the gain $\alpha > 0$ controls the learning rate (i.e., how fast $\hat{\boldsymbol{\theta}}$ is changed when the normalized estimation error $\epsilon(t)$ is non-zero and how fast the covariance matrix $\mathbf{P}(t)$ may go to zero in any direction).

The covariance matrix $\mathbf{P}(t)$ can be determined by solving
$
\dot{\mathbf{P}}(t)
        =
        -\alpha
        \frac{\mathbf{P}(t)\boldsymbol{\phi}(t)\boldsymbol{\phi}^T(t)\mathbf{P}(t)}{1+\boldsymbol{\phi}^T(t)\mathbf{P}(t)\boldsymbol{\phi}(t)}
        +
        \beta \mathbf{I} + \gamma \mathbf{P}(t) - \delta \mathbf{P}^2(t)
$,
where
the parameter $\beta > 0$ ensures adaptation does not stop by preventing $\mathbf{P}(t)$ from going to zero in any direction,
the forgetting factor $\gamma > 0$ helps account for the time variations of $\boldsymbol{\theta}(t)$ by controlling the rate at which past data is discounted,
and
the parameter $\delta > 0$ prevents $\mathbf{P}(t)$ from becoming unbounded in any direction.

This design of $\mathbf{P}(t)$ guarantees that
$0 < \mu_L \mathbf{I} \le \mathbf{P}(t) \le \mu_U \mathbf{I} < \infty$ for all $t > 0$,
with
$\mu_L = \tfrac{(\gamma-\alpha) + \sqrt{ (\gamma-\alpha)^2 + 4\delta\beta }}{2\delta}$
and
$\mu_U = \tfrac{\gamma + \sqrt{ \gamma^2 + 4\delta\beta }}{2\delta}$,
provided that the parameters are chosen such that
$0 < \mu_L < \mu_U < \infty$ 
and the initial value satisfies
$\mu_L \mathbf{I} \le \mathbf{P}(0) \le \mu_U \mathbf{I}$
 \cite{salgado1988efra,han1997efraComment}.

There are typically infinitely many choices of the parameter vector $\boldsymbol{\theta}(t)$ that satisfy (\ref{linearEstimator}), which implies $\boldsymbol{\phi}$ may not be persistently exciting and parameter drift may occur \cite{ioannou2012robust}.
To ensure $\hat{\boldsymbol{\theta}}(t)$ remains bounded, the projection function \text{prof}$(\boldsymbol{\psi})$ may be implemented as
\begin{align*}
    \begin{split}
        \text{proj}(\boldsymbol{\psi})
	    =&
	    \begin{cases}
            \boldsymbol{\psi} -\frac{\mathbf{P}(t)\hat{\boldsymbol{\theta}}(t)\hat{\boldsymbol{\theta}}^T(t)}{\hat{\boldsymbol{\theta}}^T(t)\mathbf{P}(t)\hat{\boldsymbol{\theta}}(t)}
            \boldsymbol{\psi}
            & \parbox{100pt}{ $|\hat{\boldsymbol{\theta}}(t)| = \bar{\theta}$ and
            $\hat{\boldsymbol{\theta}}^T \boldsymbol{\psi} > 0$}
	       \\
	       \boldsymbol{\psi} & \text{otherwise}
	   \end{cases}
    \end{split}
\end{align*}
The constant $\bar{\theta}$ defines the upper limit of the parameter vector's magnitude. $\bar{\theta}$ should be selected as large as feasible when prior knowledge of the disturbance frequencies is absent, ensuring that the unknown parameter vector $\boldsymbol{\theta}(t)$ remains bounded as $|\boldsymbol{\theta}(t)| < \bar{\theta}$ for all $t$.
Then we have the following result.

\begin{theorem} \label{thm:adaptiveCtrlPerformance}
    The inverted pendulum described by (\ref{sys}) with the controller described by (\ref{pdCtrl}) and (\ref{observer})-(\ref{updateLaw}) guarantees
    that $\hat{\boldsymbol{\theta}}(t)$, $\mathbf{P}(t)$, $\mathbf{P}^{-1}(t)$, and all signals in the closed-loop system are uniformly bounded.
    Furthermore, there exist $\bar{\mu}>0$ and $\bar{\varepsilon} \ge 0$ such that for all $\mu<\bar{\mu}$ we have
    $\limsup_{T\to\infty} \frac{1}{T} \int_t^{t+T}|x^d_{sc}(\tau) - x_{sc}(\tau)|^2 \partial\tau \le \rho_0\mu + \rho_1\bar{\varepsilon}$ 
    for some $\rho_0, \rho_1 \ge 0$.
\end{theorem}
\textit{Proof: see the appendix.}

\vspace{-0.05 in}
\section{CASE STUDY}
To validate the proposed control approach, we present a case study on a 2-D seven-link bipedal robot that is controlled to walk on a platform moving in both vertical and horizontal directions, as shown in Fig.\ref{fig: lip and seven link robot}(b).

The overall control framework for the actual robot employs three layers to deal with the complex full-order robot dynamics.
This framework structure has been widely used in legged robot control~\cite{xiong20223,gao2023time,gu2024walking,iqbal2023asymptotic}.
The high-level planner of the overall framework, i.e., the key novelty of this work presented in Section III, determines the ankle torque command and the desired step length based on the reduced-order pendulum model.
The middle level plans the desired task-space motion profile for the full-order robot model, which agrees with the command determined by the high-level controller.
The low-level controller produces the joint torque to track the desired task space profile provided by the middle level and execute the ankle torque command specified by the high level.




\vspace{-0.05 in}
\subsection{Full-Order Model of Robot Dynamics}
The proposed control framework design begins with full-order modeling.
Using Lagrange's equation, the full-order dynamics of the seven-link robot (Fig.\ref{fig: lip and seven link robot}(b)) is described by
\begin{align*}
    \mathbf{M}(\mathbf{q})\ddot{\mathbf{q}}
    +
    \mathbf{C}(\mathbf{q},\dot{\mathbf{q}})\dot{\mathbf{q}}
    +
    \mathbf{G}(\mathbf{q})
    =
    \mathbf{B}_a \tau_a
    +
    \mathbf{B}_b \boldsymbol{\tau}_b
    -
    \mathbf{J}^{T}(\mathbf{q}) \mathbf{F}
\end{align*}
where $\mathbf{q} \in\mathbb{R}^{9} $ is the generalized coordinates of the floating-base robot as illustrated in Figure~\ref{fig: lip and seven link robot},
$\mathbf{M}(\mathbf{q})$ is the joint-space inertia matrix, 
$\mathbf{C}(\mathbf{q},\dot{\mathbf{q}})\dot{\mathbf{q}}$ is the sum of the Coriolis and centrifugal terms, 
$\mathbf{G}(\mathbf{q})$ is the gravitational term, 
$\mathbf{J}(\mathbf{q})$ is the contact Jacobian, 
$\mathbf{F}\in\mathbb{R}^{3}$ are the contact forces and torque applied to the robot, 
$\tau_a\in\mathbb{R}$ is the torque provided by the ankle on the supporting leg,
$\boldsymbol{\tau}_b\in\mathbb{R}^{5}$ is the vector of torques applied by the rest of the motors, and
$\mathbf{B}_{a}$ and $\mathbf{B}_{b}$ are the input selection matrices for $\tau_a$ and $\boldsymbol{\tau}_b$, respectively.
We assume the robot stands on the moving surface without slipping, which means the robot motion satisfies the holonomic constraint
\begin{align*}
    \ddot{\boldsymbol{\Phi}}
    =
    \mathbf{J}(\mathbf{q})\ddot{\mathbf{q}}
    +
    \dot{\mathbf{J}}(\mathbf{q},\dot{\mathbf{q}})\dot{\mathbf{q}}
    =
    \mathbf{a}
\end{align*}
where
$\boldsymbol{\Phi} = \begin{bmatrix} x_{ak}, z_{ak}, \theta_{ak} \end{bmatrix}^T$ is the ankle pose and
$\mathbf{a} = [\ddot{x}_{ws}, \ddot{z}_{ws}, 0]^T$ is the acceleration of the supporting surface.
For notational brevity, we omit the time argument of parameters and variables in this section unless needed.

\vspace{-0.05 in}
\subsection{High-Level Controller}
 The high-level controller determines the ankle torque using the proposed adaptive controller while the step length command is computed based on the proposed footstep planner.  
 
 To simulate the effects of digital control implementation on hardware, we follow the methods in \cite{salgado1988efra, han1997efraComment, Ioannou2006tutorial} to discretize the adaptive controller (\ref{observer})-(\ref{updateLaw}) with a sampling period $T_{samp}$. Details are omitted for space consideration.

The power spectrum of the disturbance caused by surface motion is studied in \cite{tannuri2003estimating} and shows at most two peaks in the frequency range below 1 rad/s.
To prevent the compensator from attenuating $\zeta(t)$, which results in slow learning and noise amplification \cite{jafari2017overparameterized}, we select the bandwidth parameter $\sigma \gg 1$, the PD gains $k_p$, $k_d$ such that the system (\ref{pdClosedLoop}) has a bandwidth larger than 1, and a compensator order larger than twice the number of frequencies in the disturbance plus 1(i.e. $n_\phi \gg 5$).
Guidelines for tuning the parameters $\alpha$, $\beta$, $\gamma$, and $\delta$ are discussed in \cite{salgado1988efra,han1997efraComment}.
Specifically, we select $k_p = 25$, $k_d = 10$, $T_{samp} = 1/500$ s, $\sigma = 10$ rad/s, $n_\phi = 20$, and $\mathbf{P}(0) = 10^4 \mathbf{I}$.
The discrete counterparts of $\alpha$, $\beta$, $\gamma$, $\delta$, and $\bar{\theta}$ are chosen as
$ 6\times10^{-1}$, $10^{-3}$, $10^{-5}$, $10^{-6}$, $10^2$, respectively.

\vspace{-0.05 in}
\subsection{Mid-Level Motion Generation}
To ensure the seven-link robot closely emulates an inverted pendulum model and executes the commanded step length, the robot is expected to track a desired task space trajectory $ \mathbf{h}^{d} = [z_{sc}^{d}, \theta_{b}^{d}, x_{ak}^{d}, z_{ak}^{d}, \theta_{ak}^{d}]^T$. 
The values of the desired relative CoM height $z_{sc}^d$ and the desired trunk orientation $\theta_{b}^{d}$ are set to be constant to ensure the constant relative height and zero angular momentum assumptions of the reduced-order model are met.
$x_{ak}^{d}, z_{ak}^{d}$ are the desired horizontal and vertical positions of the swinging ankle. These two profiles are constructed using the $6^{\text{th}}$-order B\'ezier curve \cite{gao2023time} while $x_{ak}^{d}$ is subjected to changes according to the step length command and $z_{ak}^{d}$ is manually specified to ensure minimum impact when touchdown.
$\theta_{ak}^{d}$ is a constant value of the desired orientation of the swinging ankle to ensure flat foot touchdown with the ground.
While the robot has six motors, the dimension of $\mathbf{h}^{d}$ is only five, since the stance ankle motor is used to drive $x_{sc}$ to $x_{sc}^d$ using the adaptive control design.

\vspace{-0.05 in}
\subsection{Low-Level Controller}
The goal of the low-level controller is to drive the current task space position $ \mathbf{h}(\mathbf{q}) =: [z_{sc}, \theta_{b}, x_{ak}, z_{ak}, \theta_{ak}]^T$ to the desired task space position $ \mathbf{h}^{d}$ while ensuring the ankle torque $\tau_{a}$ is chosen according to (\ref{pdCtrl}). Using input-output linearizing control~\cite{iqbal2020provably}, we construct $\boldsymbol{\tau}_b$ as
\begin{align*}
    \begin{split}
        \boldsymbol{\tau}_b
        =
        \left(
        \mathbf{J}_{h}
        \mathbf{M}^{-1}
        \bar{\mathbf{B}}_b
        \right)
        (
        \ddot{\mathbf{h}}_d
        -
        \mathbf{K}_p \mathbf{y}
        -
        \mathbf{K}_d \dot{\mathbf{y}}
        -
        \dot{\mathbf{J}}_{h} \dot{\mathbf{q}}
        +
        \mathbf{J}_{h} \mathbf{M}^{-1} \bar{\mathbf{H}}_a )
    \end{split}
\end{align*}
where the output function $
\mathbf{y} := \mathbf{h}(\mathbf{q}) - \mathbf{h}^{d}
$ represents the task space tracking error,
$
\bar{\mathbf{H}}_a
=
\mathbf{H}_{a}
-
\mathbf{J}^T \boldsymbol{\Lambda} 
\mathbf{J} \mathbf{M}^{-1} \mathbf{H}_a
+
\mathbf{J}^T \boldsymbol{\Lambda} \dot{\mathbf{J}}\dot{\mathbf{q}}
$,
$
\mathbf{H}_{a}
=
\mathbf{C}(\mathbf{q},\dot{\mathbf{q}}) \dot{\mathbf{q}}
+
\mathbf{G}(\mathbf{q})
-
\mathbf{B}_{a} \tau_{a}
$,
$
\boldsymbol{\Lambda} 
=
\left(
\mathbf{J} \mathbf{M}^{-1} \mathbf{J}^{T}
\right) ^{-1}
$,
and
$
\bar{\mathbf{B}}_b
=
\left( \mathbf{I} - \mathbf{J}^T 
\boldsymbol{\Lambda} \mathbf{J} \mathbf{M}^{-1} \right)
\mathbf{B}_b
$.

This torque profile enforces the dynamics of the output function $\mathbf{y}$ to be a stable second-order system
$\ddot{\mathbf{y}} = 
- \mathbf{K}_d \dot{\mathbf{y}} 
- \mathbf{K}_p \mathbf{y}$
\cite{dosunmu2023stair}.
$\mathbf{K}_p$ and $\mathbf{K}_d$ are the feedback gains to be selected for sufficient output regulation performance.

\vspace{-0.05 in}
\subsection{Validation Results}
This subsection presents the simulation results of a seven-link robot walking on a moving platform.

To show the performance improvement of the proposed adaptive controller, we compare the proposed approach with (a) the HT-LIP stepping controller \cite{iqbal2023asymptotic} and (b) PD plus feed-forward control (PD+FF), as given in (\ref{pdCtrl}) with $v(t) = 0$.

Three surface motion profiles are utilized to test the performance of the proposed and the baseline controllers.
The motion profile of each case is summarized in Table~\ref{tab: motion profile}.
Case 1 is used to establish a performance baseline for each control strategy in the absence of a disturbance.
The purpose of Case 2 is to assess the controllers' effectiveness in the presence of periodic disturbances with frequency values in the ranges discussed in \cite{tannuri2003estimating} (i.e. less than 1 rad/s).
Case 3 allows the evaluation of the control approaches under disturbances with time-varying parameters.
In all cases, the desired velocity is $v^d = 0.2$ m/s, the step period is $T_s=0.5$ s, the desired vertical height is $z_{sc}^d = 0.74$ m, and the ankle torque limit is $\bar{\tau}_{a} = 40$ Nm.

\begin{table}[h!]
    \centering
    \caption{motion profile of dynamic ground}
        \vspace{-0.1 in}
    \resizebox{3.25in}{!}{%
    \begin{tabular}{|c|c|c|}
        \hline
        Case  & $x_{ws}(t)$ (m) & $z_{ws}(t)$ (m) \\ \hline
        1 & 0 & 0 \\ \hline
        2 & $0.2(1-\cos{(0.7t)})$ & $0.5(1-\cos{(0.4t)})$ \\ \hline
        3 & $0.004t^{2} \sin{(4t)} e^{-t/5}$ & $0.04(0.5\cos{(6t)} + \cos{(0.1t^2)} - 1.5)$ \\ \hline
    \end{tabular}
    \label{tab: motion profile}
    }
\end{table}

To quantitatively evaluate the performance of the controllers, we generate $N=7501$ data samples at the sample times $t_i=\frac{i-1}{500}$ s indexed by $i \in \{1,\dots,N\}$ and define the evaluation interval as $E=\{ t ~|~ 5s \le t \le 15s\}$.
The evaluation interval is chosen to start at $t=5$ s so that the controller transients have disappeared.
The following metrics are calculated and used:
\begin{enumerate}
    \item RMSE $ = \sum_{t_i \in E} \sqrt{ \frac{1}{M} |e_{sc}(t_i)|^2}$, where $M$ is the number of time samples $t_i$ in the evaluation interval $E$.
    \item PEAK $ = \max_{t_i \in E} |e_{sc}(t_i)|$.
    \item RMSE-PI $ = \sum_{t_i \in E} \sqrt{ \frac{1}{L} |e_{sc}(t_i^-)|^2 }$, where $L$ is the number of touchdown times $t_i$ within interval $E$.
    \item PEAK-PI $= \max_{t_i \in E} |e_{sc}(t_i^-)|$.
    \item TRQ $= \max_{i\in\{1,\dots,N\}} |\tau_a(t_i)|$. Note that TRQ is not applicable (NA) for the HT-LIP controller since HT-LIP does not employ ankle torque.
    \item FIT is the slope of the linear fit to $x_{s_{0}c}(t_i)$ for $t_i \in E$, where $x_{s_{0}c}(t)$ is the current horizontal position of CoM relative to the initial supporting point $s_{0}$.
\end{enumerate}
RMSE and PEAK describe the performance of the controllers over entire walking cycles while RMSE-PI and PEAK-PI describe the performance of the controller at the pre-impact time instants. 
TRQ and FIT are particularly important metrics for the robot walking problem.
The robot's ankle motors will not be able to deliver the demanded ankle torque if TRQ exceeds the ankle torque limit $\bar{\tau}_a$, and the robot achieves the desired average walking speed when the slope of the linear fit (FIT) equals the commanded velocity $v^d$.

\begin{table}[h]
\centering
\caption{Controller performance comparison}
    \vspace{-0.1 in}
\label{table:performance}
\renewcommand{\arraystretch}{1.5}
\resizebox{3.25in}{!}{%
\begin{tabular}{|c|c|c|c|c|c|c|c|}
\hline
Case               & 
Controller & 
RMSE & 
PEAK & 
RMSE-PI & 
PEAK-PI & 
TRQ & 
FIT
\\ \hline
\multirow{3}{*}{1} & HT-LIP & 8.19e-03 & 1.20e-02 & 4.79e-03 & 4.92e-03 & NA & 1.96e-01 \\
\cline{2-8}
& PD+FF & 3.52e-03 & 5.67e-03 & 5.28e-03 & 5.29e-03 & 1.46e+01 & 2.00e-01 \\
\cline{2-8}
& proposed & 1.51e-03 & 2.79e-03 & 2.18e-03 & 2.39e-03 & 1.39e+01 & 2.00e-01 \\
\hline
\multirow{3}{*}{2} & HT-LIP & 1.43e-02 & 4.00e-02 & 1.51e-02 & 2.49e-02 & NA & 1.75e-01 \\
\cline{2-8}
& PD+FF & 4.08e-03 & 9.08e-03 & 6.24e-03 & 9.15e-03 & 1.82e+01 & 2.00e-01 \\
\cline{2-8}
& proposed & 1.75e-03 & 4.17e-03 & 2.60e-03 & 4.00e-03 & 1.65e+01 & 2.00e-01 \\
\hline
\multirow{3}{*}{3} & HT-LIP & 4.39e-02 & 1.08e-01 & 6.42e-02 & 1.08e-01 & NA & 1.70e-01 \\
\cline{2-8}
& PD+FF & 1.38e-02 & 2.38e-02 & 1.38e-02 & 2.32e-02 & 4.33e+01 & 2.00e-01 \\
\cline{2-8}
& proposed & 3.09e-03 & 7.84e-03 & 2.57e-03 & 4.57e-03 & 3.43e+01 & 1.99e-01 \\
\hline
\end{tabular} 
}
\end{table}

Table~\ref{table:performance} summarizes the scores for the different combinations of controllers and cases. Note that the parameters of the baseline approaches are tuned for their best performance.
Figure \ref{fig:static} shows the performance for Case 1.
Figure \ref{fig:static} and Table \ref{table:performance} confirm that all controllers stabilizes the robot and achieves an average velocity close to the desired velocity.
The actuation provided by the ankle motor improves performance over the HT-LIP controller, and the adaptive controller further enhances the performance by compensating for the bias left behind by the PD+FF controller.
The bottom two plots in Figure \ref{fig:static} display how the different controllers adjust the control inputs to achieve the control objective.

\begin{figure}
    \centering
    \includegraphics[scale=0.9]{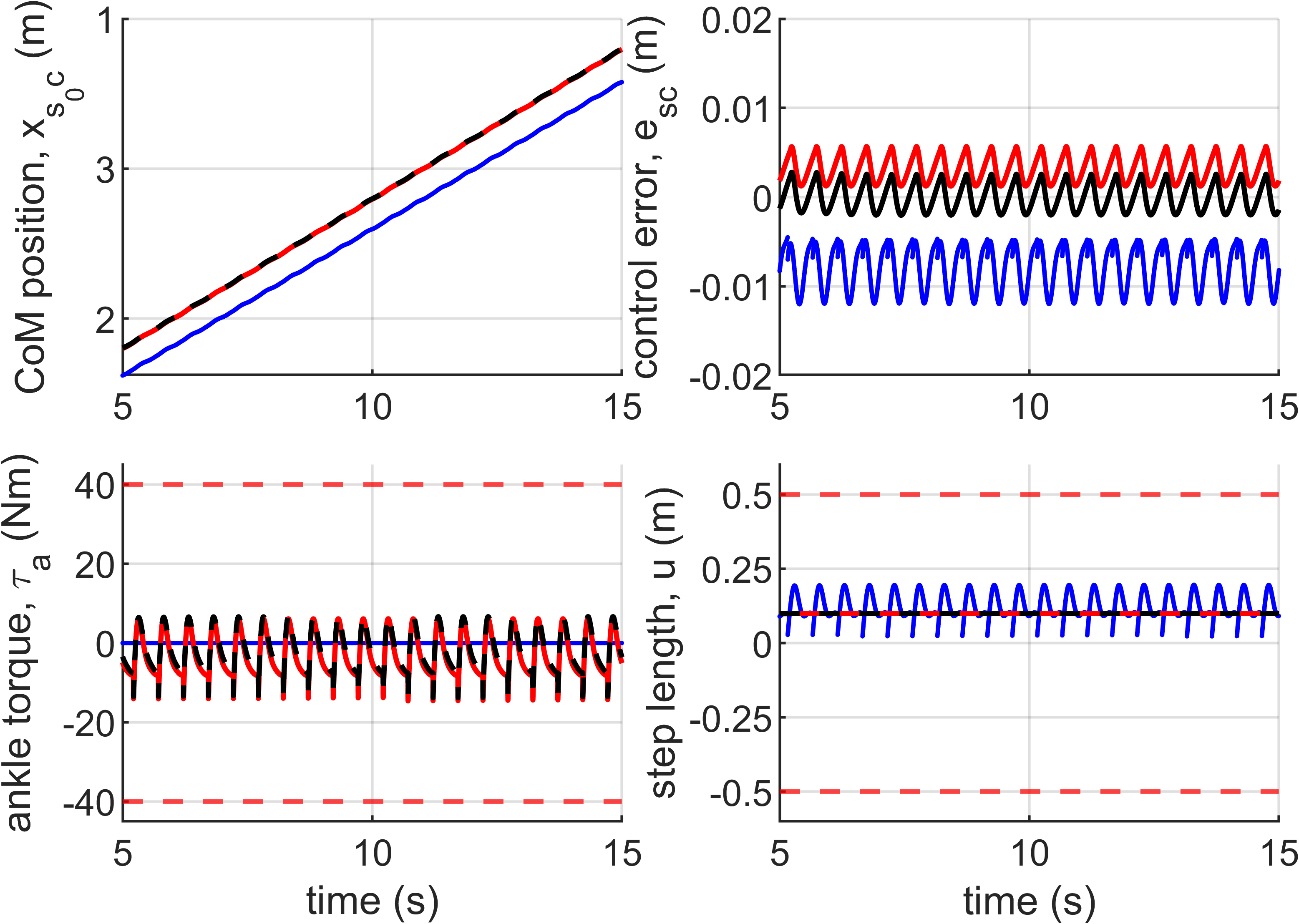}
    \vspace{-0.1 in}
    \caption{Performance of the HT-LIP (blue), PD+FF (red), and proposed (black) controllers in the absence of a disturbance under Case 1.}
    \label{fig:static}
\end{figure}

Figure \ref{fig:xcos_zcos} presents the results for Case 2.
The disturbance produces a periodic signature in Figure \ref{fig:xcos_zcos}.
The adaptive controller performs the best while the HT-LIP controller performs the worst, according to the metrics in Table \ref{table:performance}.

\begin{figure}
    \centering
    \includegraphics[scale=0.9]{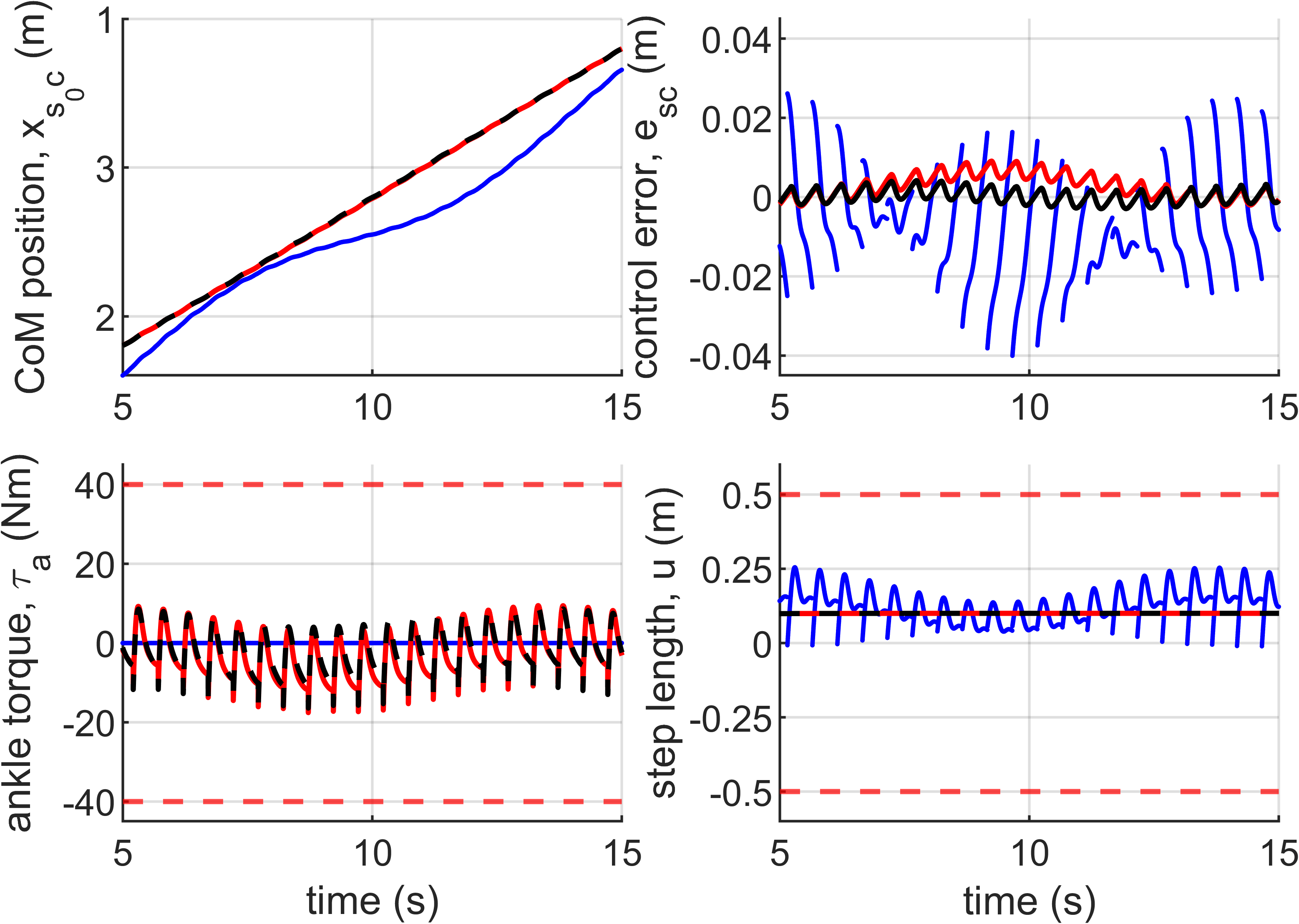}
    \vspace{-0.1 in}
    \caption{Performance of the HT-LIP (blue), PD+FF (red), and proposed (black) controllers for small amplitude disturbances under Case 2.}
    \label{fig:xcos_zcos}
    \vspace{-0.1 in}
\end{figure}



Figure \ref{fig:xz_nonper} illustrates the performance for Case 3.
The upper left plot of Figure \ref{fig:xz_nonper} and Table \ref{table:performance} indicate that the HT-LIP controller is unable to achieve the desired velocity.
The bottom left plot of Figure \ref{fig:xz_nonper} shows the PD+FF controller exceeds the ankle torque limit just prior to 10 s.
In contrast, the metrics in Table \ref{table:performance} confirm that the proposed controller not only respects the torque limit but also accurately tracks the desired position and velocity.

\begin{figure}
    \centering
    \includegraphics[scale=0.9]{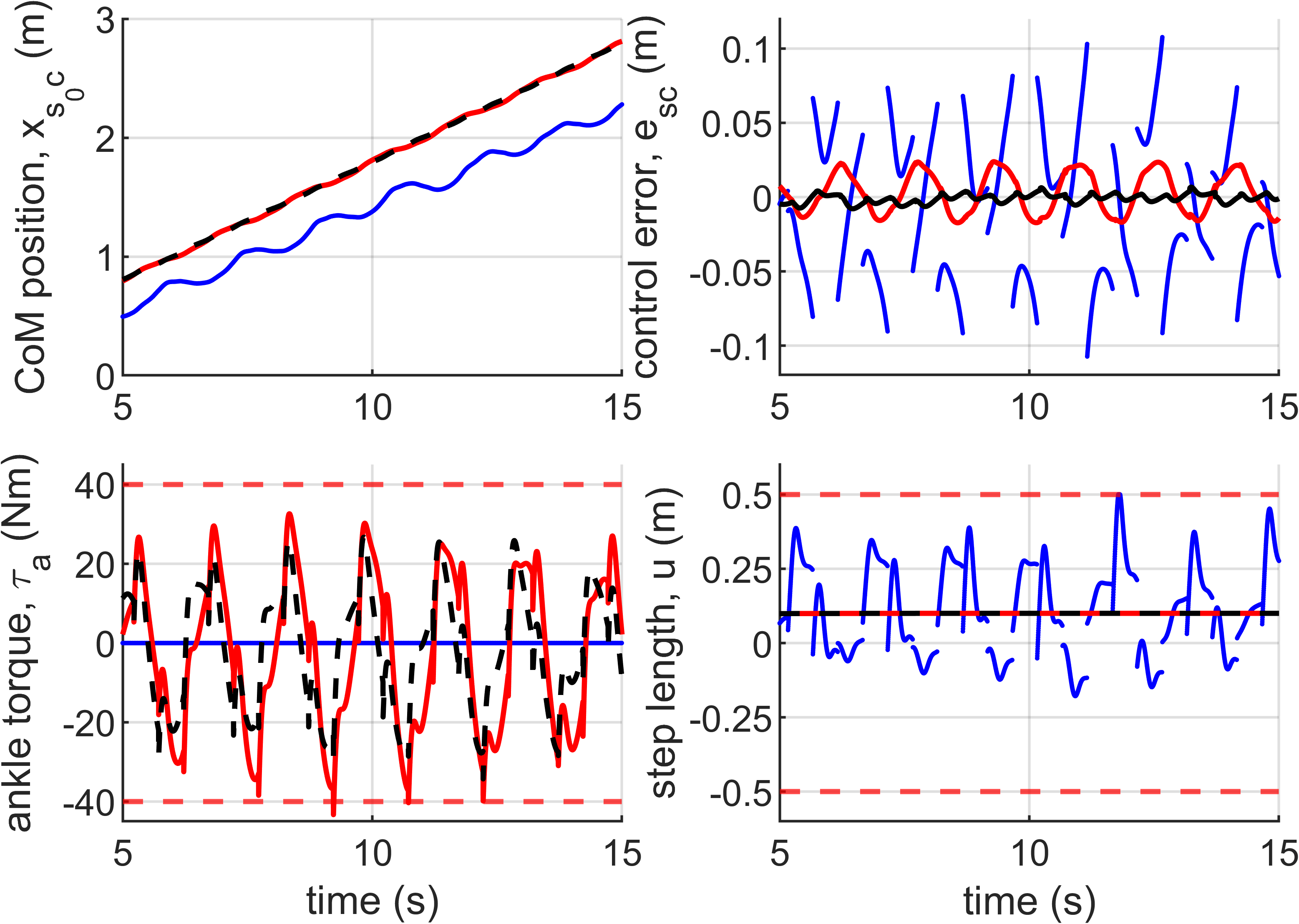}
    \vspace{-0.1 in}
    \caption{Performance of the HT-LIP (blue), PD+FF (red), and proposed (black) controllers for time-varying disturbances under Case 3.}
    \label{fig:xz_nonper}
        \vspace{-0.1 in}
\end{figure}

\section{CONCLUSIONS}

This paper has presented an adaptive control scheme to enable a bipedal robot to walk stably on a surface with unknown vertical and horizontal motions. 
By modifying the desired state trajectory, the hybrid error dynamics of the walking robot was first transformed into a continuous system, facilitating the use of adaptive control typically suited for continuous systems.
Adaptive ankle torque control was then formed to reject the unknown disturbances by addressing the continuous error system.
Additionally, a discrete footstep controller was synthesized to asymptotically stabilize the new error system, thus reducing the torque requirements of the ankle torque control.
Simulations on a seven-link biped demonstrated the superior capability of the proposed approach in achieving accurate position and respecting joint torque limits, compared to baseline methods. 
Future work will focus on extending this control scheme with a feedback law to adjust the actual step length online and conducting experimental validations on hardware.

\section*{APPENDIX}
\begin{proof}[Proof of Lemma \ref{lem:stepCtrl}]
To begin, observe that the desired profile (\ref{desProfile}) is such that the traveled distance over the step period $\Delta x^d =: x_{sc}^d(t_{k+1}^-) - x_{sc}^d(t_{k}^+)$ is equal to the desired average velocity times step period $T_s v^d$ and the boundary velocity difference $\Delta v^d =: \dot{x}^d_{sc}(t_{k+1}^-) - \dot{x}^d_{sc}(t_{k}^+)$ is zero; i.e.,
$[\Delta x^d, \Delta v^d]^T = \mathbf{A}_s \mathbf{x}^d(t_k^+) - \mathbf{x}^d(t_k^+) = (\mathbf{A}_s-I)\mathbf{x}^d(t_k^+) = [T_s v^d, 0]^T$
where $\mathbf{A}_s$ is the state-transition matrix of (\ref{desProfile}) evaluated over the $k^{th}$ step period; i.e.,
$
\mathbf{A}_s = \boldsymbol{\Phi}(t_{k+1}^- , t_k^+)
    =
    \exp( \mathbf{A}_{\lambda} T_s )
$.
Observe that the commanded profile (\ref{cmdProfile}) has the same state-matrix, $\mathbf{A}_\lambda$, as the desired profile (\ref{desProfile}); therefore, the traveled distance and the boundary velocity difference of the commanded profile $ [\Delta x^c, \Delta v^c]^T = \mathbf{A}_s \mathbf{x}^c(t_k^+) - \mathbf{x}^c(t_k^+) = (\mathbf{A}_s - \mathbf{I}) \mathbf{x}^c(t_k^+)$.
Then the commanded profile will converge to the desired profile if the step lengths $u(t_k^-)$ are designed to drive the traveled distance error and boundary velocity difference error to zero; i.e., regulate the output vector 
$\mathbf{y}_o(t_k^+) =: [\Delta x^d - \Delta x^c, \Delta v^d - \Delta v^c]^T$ 
where
$\mathbf{y}_o(t_k^+) = (\mathbf{A}_s - \mathbf{I}) \left( \mathbf{x}^d(t_k^+) - \mathbf{x}^c(t_k^+) \right) = (\mathbf{A}_s - \mathbf{I}) \mathbf{e}^c(t_k^+)$.
To obtain an equation in terms of the step lengths, observe that equations (\ref{desProfile}), (\ref{cmdProfile}), and (\ref{cmdToDesError}) may be used to show
\begin{align*}
    \begin{split}
        \dot{\mathbf{e}}^c(t) =& \mathbf{A}_\lambda \mathbf{e}^c(t)
        \\
        \mathbf{e}^c(t_k^+)
        =&
        \mathbf{e}^c(t_k^-)
        +
        \mathbf{B}_1
        \left( u(t_k^-) - T_sv^d \right)
    \end{split}
\end{align*}
These two equations may be combined to obtain
$ \mathbf{e}^c(t_{k+1}^+) = 
\mathbf{A}_s \mathbf{e}^c(t_k^+)
+ \mathbf{B}_1 \left( u(t_{k+1}^-) 
- T_sv^d \right)$.
Next, note that $\mathbf{A}_s - \mathbf{I}$ is invertible for all $T_s>0$ and 
$(\mathbf{A}_s - \mathbf{I}) \mathbf{A}_s (\mathbf{A}_s - \mathbf{I})^{-1} = \mathbf{A}_s$, 
so the dynamics of $\mathbf{e}^c(t_k^+)$ may be written as $\mathbf{y}_o(t_{k+1}^+) = \mathbf{A}_s \mathbf{y}_o(t_k^+) + \mathbf{B}_s \left( u(t_{k+1}^-) - T_sv^d \right)$ 
where $\mathbf{B}_s = (\mathbf{A}_s - \mathbf{I}) \mathbf{B}_1$.
Observe that the pair $( \mathbf{A}_s , \mathbf{B}_1 )$ is controllable for all $T_s>0$.
As a result, the pair $(\mathbf{A}_s, \mathbf{B}_s)$ will be controllable because it is related to $(\mathbf{A}_s, \mathbf{B}_1)$ by the similarity transform $\mathbf{A}_s - \mathbf{I}$ \cite{hespanha2018linear}.
Then  choosing the step lengths according to the LQR control law (\ref{LQRsteplengthcontrollaw}) will stabilize the system and drive $\mathbf{x}^c \to \mathbf{x}^d$ \cite{hespanha2018linear}.
\end{proof}

\begin{proof}[Proof of Lemma \ref{lem:stable}]
    To begin, consider the Lyapunov candidate $V = \mathbf{e}^T \mathbf{L} \mathbf{e}$.
    Taking the derivative of $V$ and using (\ref{pdClosedLoop}) with $v(t) = 0$ yields
    \begin{align*}
        \dot{V}
        \le&
        -\left( 1 - 2\left|\tfrac{\ddot{z}_{ws}}{z_{sc}}\right|\| \mathbf{L} \| \right) |\mathbf{e}|^2
        +
        2 \| \mathbf{L} \| |\mathbf{e}| \left| \ddot{x}_{ws} - \tfrac{ x_{sc}^c }{z_{sc}}\ddot{z}_{ws} \right|
        \\
        \le&
        -\left( 1 - 2\rho_1\| \mathbf{L} \| \right) |\mathbf{e}|^2
        +
        2 \| \mathbf{L} \| |\mathbf{e}| \rho_0
    \end{align*}
    where 
    $\rho_0 = \sup_t \left| \ddot{x}_{ws}(t) - \frac{ x_{sc}^c(t) }{z_{sc}(t)}\ddot{z}_{ws}(t) \right|$ 
    and 
    $\rho_1 = \sup_t \left| \frac{\ddot{z}_{ws}(t)}{z_{sc}(t)} \right|$.
    Observe that the assumption $\sup_t \left| \frac{\ddot{z}_{ws}(t)}{z_{sc}(t)} \right| < \frac{1}{2\| \mathbf{L} \|}$ implies $1 - 2\rho_1\| \mathbf{L} \| > 0$, so we may follow the approach used in \cite{isidori2017lectures} to show the error will converge exponentially fast to a residual set.
    Assume we would like to enforce $\dot{V} \le -\epsilon V$ for some $\epsilon \in \left( 0 , 1/\| \mathbf{L} \| - 2\rho_1 \right)$.
    This condition will be satisfied if
    $-\left( 1 - 2\rho_1\| \mathbf{L} \| \right) |\mathbf{e}|^2 + 2 \| \mathbf{L} \| |\mathbf{e}| \rho_0 \le -\epsilon \| \mathbf{L} \| |\mathbf{e}|^2 \le -\epsilon V$.
    Rearranging and dividing by $\left( 1 - 2\rho_1\| \mathbf{L} \| -\epsilon \| \mathbf{L} \| \right)|\mathbf{e}|$ yields
    $\frac{2\| \mathbf{L} \|\rho_0}{1 - 2\rho_1\| \mathbf{L} \| -\epsilon \| \mathbf{L} \|} \le |\mathbf{e}| $
    Therefore, $\dot{V} \le -\epsilon V$ for all $|\mathbf{e}| \ge \frac{2\| \mathbf{L} \|\rho_0}{1 - 2\rho_1\| \mathbf{L} \| -\epsilon \| \mathbf{L} \|}$ and $\epsilon \in \left( 0 , 1/\| \mathbf{L} \| - 2\rho_1 \right)$, which implies that the system will converge exponentially fast to the residual set where
    $|\mathbf{e}| \le \frac{2\| \mathbf{L} \|\rho_0}{1 - 2\rho_1\| \mathbf{L} \|}$ \cite{isidori2017lectures}.
    In the limit as $t\to\infty$ this becomes
    $|\mathbf{e}| \le \frac{2\| \mathbf{L} \|}{1 - 2\rho_1\| \mathbf{L} \|}\left(\bar{x}+\frac{0.5T_sv^d}{2\|\mathbf{L}\|}\right)$ because Lemma \ref{lem:stepCtrl} shows $x_{sc}^c\to x_{sc}^d$ as $t\to\infty$, which implies $\limsup_{t\to\infty}|x_{sc}^c(t)|=\sup_t|x_{sc}^d(t)|=0.5T_sv^d$.
    This condition becomes $|\mathbf{e}|\le0$ in the absence of a disturbance thus the error will converge to zero exponentially fast in the absence of a disturbance and model uncertainty.
\end{proof}

\begin{proof}[Proof of Theorem \ref{thm:adaptiveCtrlPerformance}]
    To begin, we show the disturbance (\ref{dist}) may be modeled as a sum of sinusoidal components with time-varying parameters plus an arbitrarily small model error and a decaying term.
    Observe that the desired profile (\ref{desProfile}) has the solution
    $x_{sc}^d(t) = \tfrac{T_s v^d}{2 \sinh (0.5\lambda T_s) } \sinh (\lambda s)$
    where 
    $s = \text{mod}(t-0.5T_s,T_s)$.
    The Fourier series expansion \cite{Korner_1988} of the desired profile is
    $x_{sc}^d(t) = \sum_{k=1}^{\infty}\xi_k \sin\left( \tfrac{2\pi k}{T_s} t \right)$
    where the amplitudes are
    $\xi_k = \tfrac{4\pi v^d T_s k }{\lambda^2 T_s^2 + 4\pi^2k^2} \cos(\pi (k-1))$.
    Separating $n_d$ terms of the Fourier series of $x_{sc}^d(t)$ yields
    $\ddot{z}_{ws}(t) x_{sc}^c(t)
    = 
    \ddot{z}_{ws}(t)e_{sc}^c(t) 
    - 
    \ddot{z}_{ws}(t)\sum_{k=1}^{n_d}\xi_k \sin\left( \tfrac{2\pi k}{T_s} t \right)
    -
    \ddot{z}_{ws}(t)\sum_{k=n_d+1}^{\infty}\xi_k \sin\left( \tfrac{2\pi k}{T_s} t \right)$.
    This result may be used with (\ref{distAccel}) and the trigonometric product to sum identities to obtain 
    $d(t) =
    \tfrac{1}{k_p}\ddot{x}_{ws}(t)
    -
    \tfrac{1}{k_p}\sum_{k=1}^{n_d}\tfrac{a_{z0}(t)}{z_{sc}(t)} \xi_k \sin\left( \tfrac{2\pi k}{T_s}t \right)
    -
    \sum_{k=1}^{n_d}\sum_{i=1}^{n_z} \tfrac{\xi_k a_{zi}(t)}{2 k_p z_{sc}(t)} \left( c^-_{ik}(t) - c^+_{ik}(t) \right)
    +
    \tfrac{\ddot{z}_{ws}(t)}{k_pz_{sc}(t)}e^c_{sc}(t)
    +
    \varepsilon(t)$
    where 
    $\varepsilon(t) = \tfrac{\ddot{z}_{ws}(t)}{z_{sc}(t)k_p}\sum_{k=n_d+1}^{\infty}\xi_k \sin\left( \tfrac{2\pi k}{T_s} t \right)$, 
    $c^-_{ik}(t) = \cos\left( \int_0^t \left(\omega_{zi}(\tau) - \tfrac{2\pi k}{T_s}\right) \partial\tau + \varphi_{zi} \right)$, and
    $c^+_{ik}(t) = \cos\left( \int_0^t \left(\omega_{zi}(\tau) + \tfrac{ 2\pi k}{T_s} \right)\partial\tau + \varphi_{zi} \right)$.
    $n_d$ may be made large enough to satisfy the bound $\bar{\varepsilon}$ because the magnitude of $\xi_k \to 0$ like $1/k$ as $k\to\infty$.
    Next, observe that the system (\ref{pdClosedLoop}) may be described using input-output operators \cite{ioannou1993ltvSystems} as
    $e_{sc}(t) = G(s,t)[v(t)]-G(s,t)[d(t)]$
    where the plant model $G(s,t)$ is given by
    $G(s,t) = \left( s^2 + k_ds + k_p + \tfrac{\ddot{z}_{ws}(t)}{z_{sc}(t)} \right)^{-1}k_p$
    and $s=\partial/\partial t$ is the differential operator.
    The internal model \cite{nikiforov1998adaptive,liu2009parameter,marino2002global} of the sinusoidal components of the disturbance will exist because the parameters of the sinusoidal components are assumed to be slowly varying and we may make the frequencies distinct and amplitudes non-zero by combining terms with the same frequencies and eliminating terms with zero amplitude \cite{stewart2024ltvSystem}.
    Then the disturbance rejection problem will be solvable because the internal model of the disturbance exists and the plant model $G(s,t)$ is of the form $G(s,t)=R^{-1}(s,t)Z(s,t)$ where $R(s,t)$ is a stable polynomial differential operation and $Z(s,t)$ is a constant \cite{stewart2024ltvSystem}.
    The proposed adaptive scheme is robust to the bounded uncertainty $\varepsilon(t)$ and can use the simplified LTI model $G_0(s)= (s^2 + k_ds + k_p)^{-1}k_p$ which omits the unknown acceleration $\ddot{z}_{ws}(t)$ \cite{jafari2016contMimo}.
    As a result of these observations, we conclude there exists constants $\bar{\mu}>0$ and $\rho_0,\rho_1>0$ such that for all $\mu < \bar{\mu}$ the output $e_{sc}(t)$ satisfies
    $\limsup_{T\to\infty} \tfrac{1}{T}\int_t^{t+T} |e_{sc}(\tau)|^2\partial \tau \le \rho_0 \mu + \rho_1\bar{\varepsilon}$.
    From Lemma \ref{lem:stepCtrl} $x^c\to x^d$, so $\limsup_{T\to\infty} \tfrac{1}{T}\int_t^{t+T} |x_{sc}^d(\tau) - x_{sc}(\tau)|^2\partial \tau \le \rho_0 \mu + \rho_1\bar{\varepsilon}$.
\end{proof}


\bibliographystyle{IEEEtran}
\bibliography{Reference.bib}


\addtolength{\textheight}{-12cm}   


\end{document}